\documentclass[letterpaper, 10 pt, conference]{ieeeconf}  

\IEEEoverridecommandlockouts                              

\usepackage{graphicx} 
\usepackage{amsmath} 
\usepackage{amssymb}  
\usepackage{url}
\usepackage[hidelinks]{hyperref}
\usepackage[ruled,vlined]{algorithm2e}
\LinesNumbered
\SetKwIF{If}{ElseIf}{Else}{if}{}{else if}{else}{end if}%
\SetKwFor{While}{while}{}{}%
\SetKwFor{For}{for}{}{}%
\SetKwFor{If}{if}{}{}%
\SetKwRepeat{Do}{do}{while}
\SetKw{KwGoTo}{go to}

\graphicspath{ {./img/} }

\newtheorem{definition}{Definition}
\newtheorem{lemma}{Lemma}
\newtheorem{corollary}{Corollary}

\title{\LARGE \bf
Learning a Model for Inferring a Spatial Road Lane Network Graph using Self-Supervision
}

\author{Robin Karlsson$^{1,2*}$, David Robert Wong$^{1,2}$, Simon Thompson$^{1}$, and Kazuya Takeda$^{1,2,3}$%
\thanks{$^{1}$Tier IV Inc., Tokyo, Japan.}%
\thanks{$^{2}$Institute of Innovation for Future Society, Nagoya University, Aichi, Japan.}%
\thanks{$^{3}$Graduate School of Informatics, Nagoya University, Aichi, Japan.}%
\thanks{*Corresponding author: {\tt\small robin.karlsson@tier4.jp}}%
\thanks{Code repository: \url{https://github.com/tier4/road_lane_network_graph_open}}%
}

\begin{document}

\maketitle
\thispagestyle{empty}
\pagestyle{empty}

\begin{abstract}

Interconnected road lanes are a central concept for navigating urban roads. Currently, most autonomous vehicles rely on preconstructed lane maps as designing an algorithmic model is difficult. However, the generation and maintenance of such maps is costly and hinders large-scale adoption of autonomous vehicle technology. This paper presents the first self-supervised learning method to train a model to infer a spatially grounded lane-level road network graph based on a dense segmented representation of the road scene generated from onboard sensors. A formal road lane network model is presented and proves that any structured road scene can be represented by a directed acyclic graph of at most depth three while retaining the notion of intersection regions, and that this is the most compressed representation. The formal model is implemented by a hybrid neural and search-based model, utilizing a novel barrier function loss formulation for robust learning from partial labels. Experiments are conducted for all common road intersection layouts. Results show that the model can generalize to new road layouts, unlike previous approaches, demonstrating its potential for real-world application as a practical learning-based lane-level map generator.

\end{abstract}

\section{INTRODUCTION}

The concept of road lanes is central for safe and efficient sharing of the road by multiple road users, and by knowing how lanes are connected, it is possible to navigate the road according to traffic rules and conventions using existing motion planning methods \cite{Paden2016, Claussmann2019}. However, inferring the directional connectivity between lanes is a non-trivial task.

Most autonomous vehicle (AV) systems rely on pre-constructed, high-definition (HD) maps \cite{Sheif2016} for tasks such as ego-localization and path planning. These maps use a graphical representation of the road scene to enable sensor-based localization \cite{Kuutti2018, Wolcott2017, Levinson2010}, efficient rule-based methods \cite{Urmson2009, Nilsson2016} or learned methods \cite{Bansal2018, Schulz2018} for motion planning, as well as being advantageous for the prediction task \cite{Djuric2020}. Mapping is still a challenging aspect of AV deployment in the real world, with limited consensus on what kind of HD map to use and even what data and sensors should be employed in the localization and path planning pipeline \cite{Sheif2016}. The variety of map types, and the fact that HD map construction still usually relies on labor-intensive human annotation, make the mapping process difficult to scale. It remains unclear when, how, and if large regions will be mapped with lane-level information for commercial use, thus tending to limit the adoption of AV technology to small geofenced and predetermined routes. In addition, the dependence on the HD map in the ego-localization and motion planning pipeline renders the vehicle unusable when either an HD map is unavailable for a particular road to be traversed, or when a failure in any of the map loading or localization systems occurs. This can also become apparent when changes in road layout, construction, and environmental changes affect the performance of localization and planning. These problems illustrate the importance of AV technology which allows a vehicle to navigate encountered road scenes based on information from onboard sensors only.

\begin{figure}[t]
  \centering
  \includegraphics[width=0.49\textwidth]{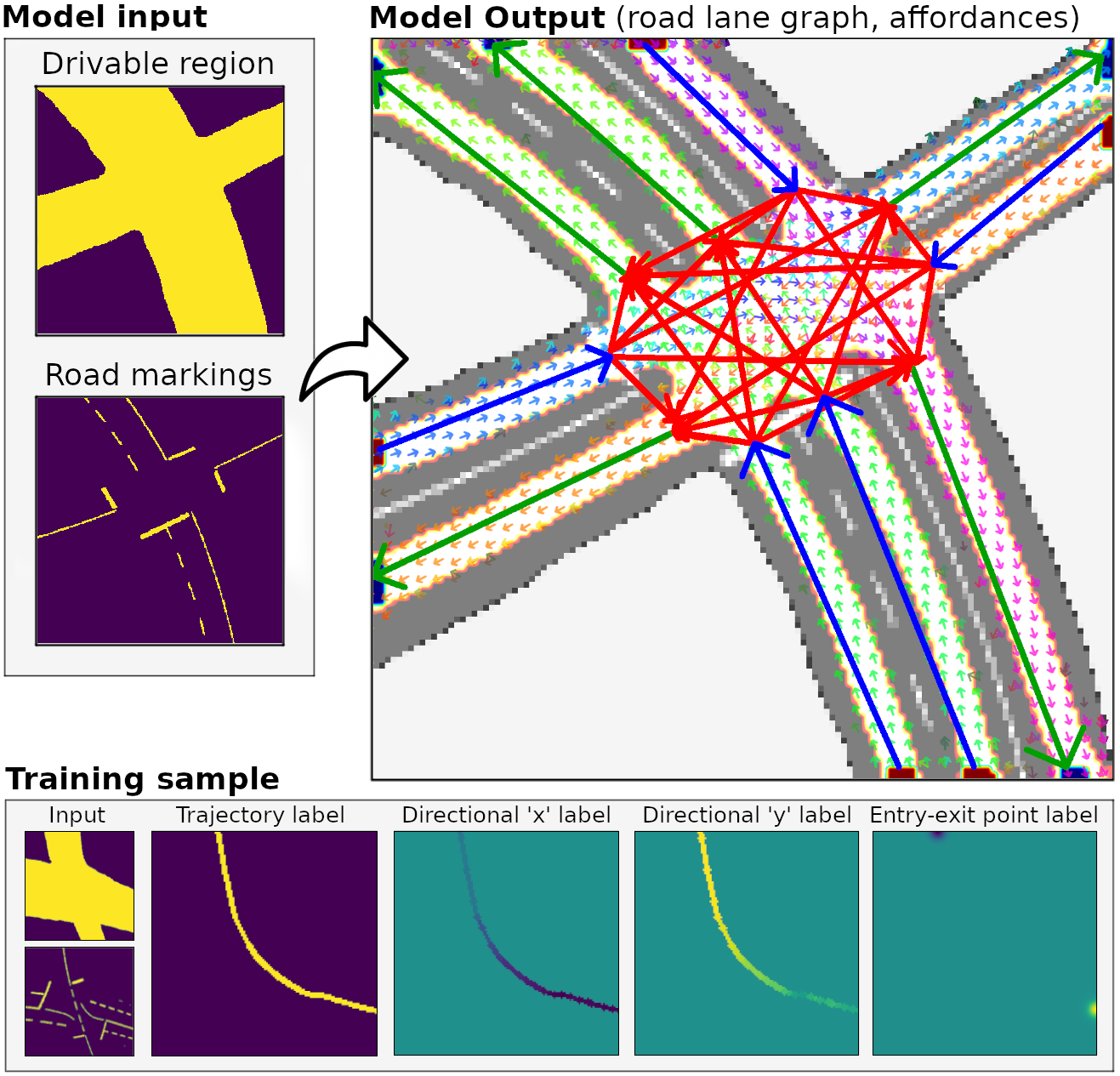}
  \caption{The hybrid neural and search-based graph model infers the spatial connectivity of road lanes (graph output) and directional lane affordance (dense affordance map output) for a road scene represented by drivable regions and road markings. Training samples consist of two input layers coupled with one example trajectory label represented as four dense layers. A self-supervised learning method enables the model to learn to infer the complete road lane network graph from single trajectory examples. The model generalizes to new road layouts and partially observed road scenes.}
  \label{fig:gdsla_promo_img}
\end{figure}

Unlike conventional AVs, humans can infer road lanes and their directional connectivity based on contextual features observed in the road scene, having learned to generalize patterns observed from example trajectories in similar road scenes. However, designing an algorithmic model to infer the lane-level road network graph based on observable features is difficult due to regional differences, local feature variation, and noise, as exemplified by varying or missing road markings, curbs, road geometries, and surface materials.

This paper presents the first self-supervised learning method for training a model to infer a spatially grounded directional road lane network graph based on a dense segmented representation of the road scene that can be generated online from onboard sensor data~\cite{Salzmann2019}. The graph can have many uses; as an online substitute for an HD map, to evaluate the consistency of the onboard map with the observed environment, as well as repairing the map. Training samples consist of single trajectories driven through a road scene, which can be automatically generated from human driving data without human labeling effort, allowing learning from an easily obtainable large and varied real-world dataset.

Our novel contributions are threefold:
\begin{itemize}
 \item The first self-supervised learning method for training a model to output a \emph{spatially grounded road lane network graph} based on dense segmentation maps representing the road scene. The method learns how to infer the complete graph from single isolated driving trajectory examples without requiring human labeling effort.
 \item A novel \emph{barrier function loss formulation for robust learning from partial labels}, demonstrating practically perfect true positive road lane output (99.93\% test accuracy), while keeping the number of false positives low, resulting in distinctly separated lanes.
 \item A \emph{formal road lane network model} proving that any structured road scene can be represented by a directed acyclic graph (DAG) of at most depth three while retaining the notion of intersection regions, and that this is the most compressed representation. The presented hybrid model implements the formal model.

\end{itemize}

\section{Related work}
\label{sec:related_work}

\subsection{Automatic and semi-automatic lane-level map generation}
\label{sec:automatic_lane_level_map_generation}

Recent works have presented approaches for reducing the amount of human labor needed for map generation. Iesaki et al.~\cite{Iesaki2019} presented a method for generating polynomial curves between intersection road lanes, fitted according to a learned cost function. Zhao et al.~\cite{Zhao2019} presented a SLAM-based method to generate a closed vector map including road lanes. Guo et al.~\cite{Guo2016} presented a method for generating a lane-level road network graph based on superimposed vehicle trajectories and road markings. Road lane connectivity in intersections is inferred from trajectories, and connecting paths are fitted using a heuristic. Neither of these approaches learns contextual features corresponding to lane connectivity, and thus cannot generalize beyond the road scenes the models have been trained on. Additionally, previous models as represented by \cite{Guo2016} tend to use heuristics that depend on a particular feature such as road markings. Our work extends upon these studies by introducing a method that can generalize to new road scenes by learning general contextual features extracted from online sensor data.

Homayounfar et al.~\cite{Homayounfar2018} presented a fully automatic approach for inferring discrete road lanes in highway road scenes as polylines using a recurrent neural network model. An extension of this work~\cite{Homayounfar2019} also introduces forking and merging road lane topology, resulting in a DAG road lane model. Our work further extends their work by demonstrating applicability to road scenes including complex intersections, meaning our approach is applicable for any general structured road scene. Additional merits of our method are that it is self-supervised, formally proved to infer the most compressed representation of the road lane network graph, as well as better suitability for real-time application by not being a recursive model.

\subsection{HD map free road nagivation}

Related works include Salzmann et al.~\cite{Salzmann2019} who trained a probabilistic CNN model to output dense ego-vehicle path affordance from recorded driving trajectories with a weighted mask loss. Amini et al.~\cite{Amini2019} demonstrated a mapless driving approach for urban roads using a variational neural network (NN) model to navigate an AV.  Ort et al.~\cite{Ort2018} demonstrated an approach for topologically simple rural roads based on combining driving trajectories with a sparse map containing road-level traffic rules. P\'{e}rez-Higueras et al.~\cite{PerezHigueras2018} trained a CNN model to infer a multimodal path conditioned on a start and goal point in a grid map environment. Training labels consist of human example trajectories and the output is used to guide an RRT* planner. Barnes et al.~\cite{Barnes2016} demonstrated a self-supervised learning method to train a CNN model to generate dense multimodal paths for an ego-vehicle. While these methods have been shown successful in relatively unconstrained environments, it is not clear how to apply these methods to complicated intersections requiring navigation constrained by traffic rules. Our model adds to overcoming these limitations by inferring a discrete graphical representation that can be used by conventional motion planning algorithms.

\subsection{Road scene understanding}

Recent studies include Wang et al.~\cite{Wang2019} who trained a model to generate top-down semantic representations of the road scene for high-level decision making. However, the representation is not spatially grounded nor does it contain lane connectivity and thus cannot be used directly for local planning. Kunze et al.~\cite{Kunze2018} presented a model for generating a hierarchical lane-level scene graph based on road marking and curb detection. The graph is not spatially grounded and thus has the same limitations as \cite{Wang2019}. Geiger et al.~\cite{Geiger2014} presented a model which infers a dense semantic directional lane affordance representation of the road scene and object poses, based on image sequences of other vehicles moving in the scene. However, the output fidelity is low and relies on observing human driving trajectories in the operating environment instead of learning contextual features as in our proposed model.

\section{Hybrid model implementation}
\label{sec:road_network_inference_model}

One can define many ways to represent the road lane network graph. For example, Homayounfar et al.~\cite{Homayounfar2019} represents the graph densely as sequentially inferred equidistant points. On the other extreme, one can imagine only connecting road lane entry and exit points to represent the graph.

To answer the question "what is the best representation", we present a formal road lane network graph model capable of expressing the road lane connectivity of any structured road scene. The formal model is found to correspond to a directed acyclic graph (DAG) of maximum depth 3, and is formally proved to be the most compressed graph able to represent the lane network of any road scene while retaining the notion of intersection regions. The formal model along with proofs are presented in the Appendix.

The formal road lane network model is implemented as a hybrid two-stage model depicted as a block diagram in Fig.~\ref{fig:hybrid_model}. The flow of the hybrid model is as follows. First, dense input features representing the road scene are generated. The input features are feed to a neural deep learning model (Sec.~\ref{sec:neural_model}), which outputs a set of dense road lane affordance maps, a multimodal directional field, and road lane entry and exit points. Finally, a road lane graph describing the global connectivity of the road scene is generated by a search-based graph generation model (Sec.~\ref{sec:graph_model}) constrained by the dense affordance map output.

\begin{figure}
\centering
\includegraphics[width=0.48\textwidth]{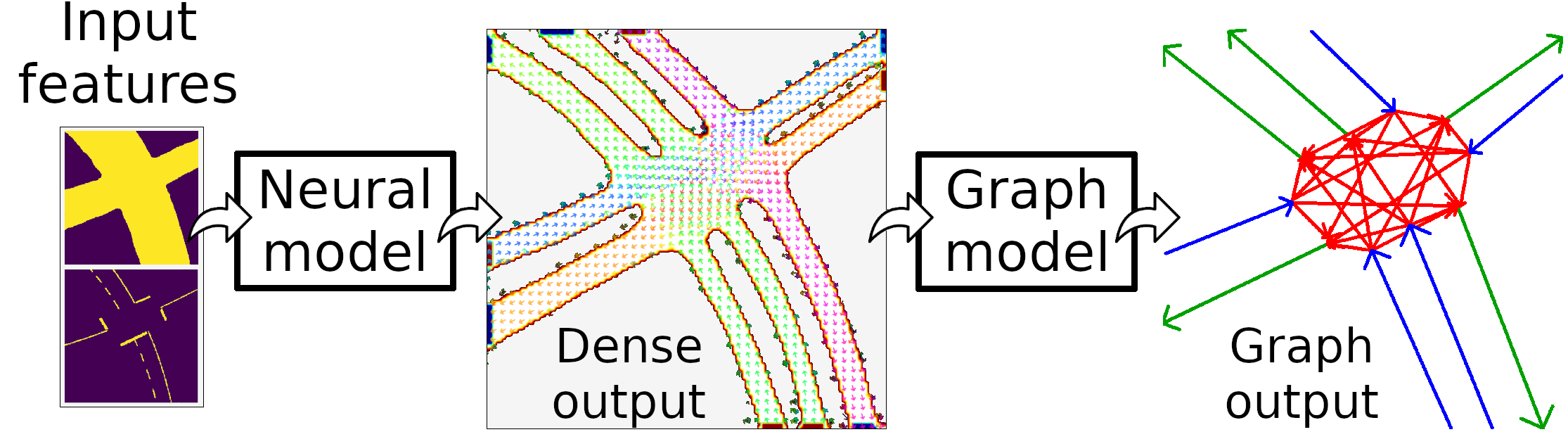}
\caption{The hybrid model consists of two sequential models. The first neural model (Sec.~\ref{sec:neural_model}) takes in a set of input features and outputs a set of dense affordances representing road lanes, multimodal directionality, and road lane entry and exit points. The second graph model (Sec.~\ref{sec:graph_model}) generates a road lane network graph based on the dense output.}
\label{fig:hybrid_model}
\end{figure}

The input road scene representation consists of a two-layered 256x256 segmented drivable region and road marking top-down view tensor, spanning a region large enough to encompass a single intersection as exemplified by Fig.~\ref{fig:gdsla_promo_img}. The value of each grid map point $(i,j)$ is a probabilistic measure where positive, unknown, and negative observations equal 1, 0.5, and 0, respectively. Previous work demonstrates that such input representations can be generated in real-time based on onboard camera and/or lidar data \cite{Salzmann2019}.

\section{Neural model}
\label{sec:neural_model}

The first part of the hybrid model depicted in Fig.~\ref{fig:hybrid_model} generates dense affordance map output that is used by the graph model (Sec.~\ref{sec:graph_model}) to generate the road lane network graph. Input features representing dense semantic information of the road scene are fed to the model, which outputs a dense intermediate representation of the directional road lanes, predicting the local structure of the road scene. The neural model is based on the Directional Soft Lane Affordance (DSLA) model for road scene understanding published by Karlsson et al.~\cite{Karlsson2020}. The multi-task network architecture consists of a single encoder-decoder network with separate output heads for each output type.

Three dense affordance networks heads each output a probability value $y_{i,j}$ for every grid map point $(i,j)$ being part of a road lane, lane entry point, and lane exit point. Directionality is modeled by a probabilistic von Mises mixture density network, predicting the directionality of the road at $(i,j)$ as a multimodal directional distribution $p_{i,j}(\theta)$ for $\theta~\in~[0, 2 \pi]$. Three directional components are found to be sufficient for representing road lane directionality. However, the model can be modified to output more modalities or even an infinite number of modalities through a sampling-based implementation \cite{Prokudin2018}. Each directional mode is visualized as a small arrow in Fig.~\ref{fig:gdsla_promo_img}.

\subsection{Self-supervised neural training process}
\label{sec:self_supervised_training_process}

This section describes the self-supervised method used to train the neural model. Self-supervision is implied in the sense that the training data is generated automatically without requiring human labeling, as similarly demonstrated by Nava et al.~\cite{Nava2019} and Barns et al.~\cite{Barnes2016}. Self-supervised methods possess considerable real-world application advantages over conventional supervised methods, as the cost and human effort of obtaining training data can be orders of magnitude lower compared with human-annotated labels.

The partial label tensor consists of a 4x128x128 tensor; a trajectory label encoded by a binary mask, two directional 'x' and 'y' labels encoding trajectory direction as element-wise unit vectors $(\hat{n}_x, \hat{n}_y)$, as well as the trajectory entry and exit locations represented as dense points. Labels can be generated from driving data by pairing a representation of the road scene with a path of a vehicle going through it. A training sample is visualized in Fig.~\ref{fig:gdsla_promo_img}.

\subsubsection{Dense affordance learning}

The dense affordance networks that predict road lane, entry point, and exit point affordances $y$ from partial labels $\hat{y}$ are trained by a novel barrier function loss formulation given in Eq.~(\ref{eq:barrier_function_loss}). The new loss function $L_{aff}$ displays favorable characteristics for learning categorical predictions from partial labels compared with the masked L2 loss formulation used in~\cite{Karlsson2020}, such as improved generalization ability and lack of occasional discontinuities in road lane affordance map output.

\begin{equation} \label{eq:barrier_function_loss}
 L_{aff} = \frac{1}{N} \left ( \sum_{i,j}  \lvert \hat{y} - y \rvert - \alpha \: CE(y,\hat{y}) \right)
\end{equation}

The underlying principle of the barrier loss formulation is that the model prediction elements $y_{i,j}$ should at least always correspond to known positive label elements $\hat{y}_{i,j}$, and as such the model is heavily penalized by a theoretically infinite loss when mispredicting elements encompassed by the partial label. This is achieved by multiplying the cross entropy term by a large value $\alpha = 10^5$. Gradient clipping is utilized to prevent exploding gradients. The unknown elements of the label $\hat{y}$ penalizes the model output $y$ through an L1 loss term, which is known to be inherently robust to uncertainty and noisy labels~\cite{Manwani2013}. In case both positive and negative partial label elements are available, an additional cross entropy term for the known negative elements can be added.

Parallels can be drawn between the process of learning with a barrier function as loss, and the version space algorithm in symbolic machine learning~\cite{Mitchell1982} that reduces an a priori set of all possible hypotheses to a subset of hypotheses satisfying all observed examples. For the case of learning road lane affordance with partially known positive labels (i.e. an example trajectory), and lack of any known negative labels, the model initially learns to trivially predict all elements as positive, analogous to starting with the initial set of all possible hypotheses in the version space algorithm. The L1 loss component gradually drives the model output towards a default prediction value which is negative in this work, without mispredicting known positives and thus avoids incurring a large penalty by the barrier function loss, which is analogous to gradual hypothesis reduction in the version space algorithm. Further mathematical analysis of this intuitive similarity is believed to reveal why the proposed barrier function loss formulation is so effective at learning and generalizing from partial labels compared with simply increasing the overall learning rate.

\subsubsection{Directional affordance learning}

The second loss formulation involves minimizing the KL divergence loss for every grid map point $(i,j)$ encompassed by the trajectory label following~\cite{Karlsson2020}.

\begin{equation} \label{eq:da_loss}
 L_{DA} = \frac{1}{N} \sum_{i,j \in mask} D_{KL} ( p(\theta)_{i,j} || \hat{p} (\theta)_{i,j} )
\end{equation}

The multimodal directional distribution prediction $p(\theta)_{i,j}$ is generated by three output network heads. The directional label is obtained from the example trajectory, resulting in a monomodal distribution $\hat{p}(\theta)_{i,j}$, approximating the true multimodal directionality at the point.

\section{Graph model}
\label{sec:graph_model}

The final part of the hybrid model depicted in Fig.~\ref{fig:hybrid_model} generates the spatially grounded discrete road lane network graph based on the dense affordance predicted by the neural model. First, a weighted adjacency matrix $A$ is generated, representing the navigational constraints imposed by the dense affordance output, defining the cost function for the A$^*$ search algorithm~\cite{Hart1968} used to search for valid paths connecting entrance and exit points. Secondly, all found paths are decomposed into a set of entrance, intersection, and exit lanes by unifying paths corresponding to the same road lane. The resulting graph is a DAG of maximum depth three, satisfying the requirements stipulated by the formal model defined in the Appendix. The rest of this section explains each step in more detail.

\begin{figure}[t]
\centering
\includegraphics[width=0.25\textwidth]{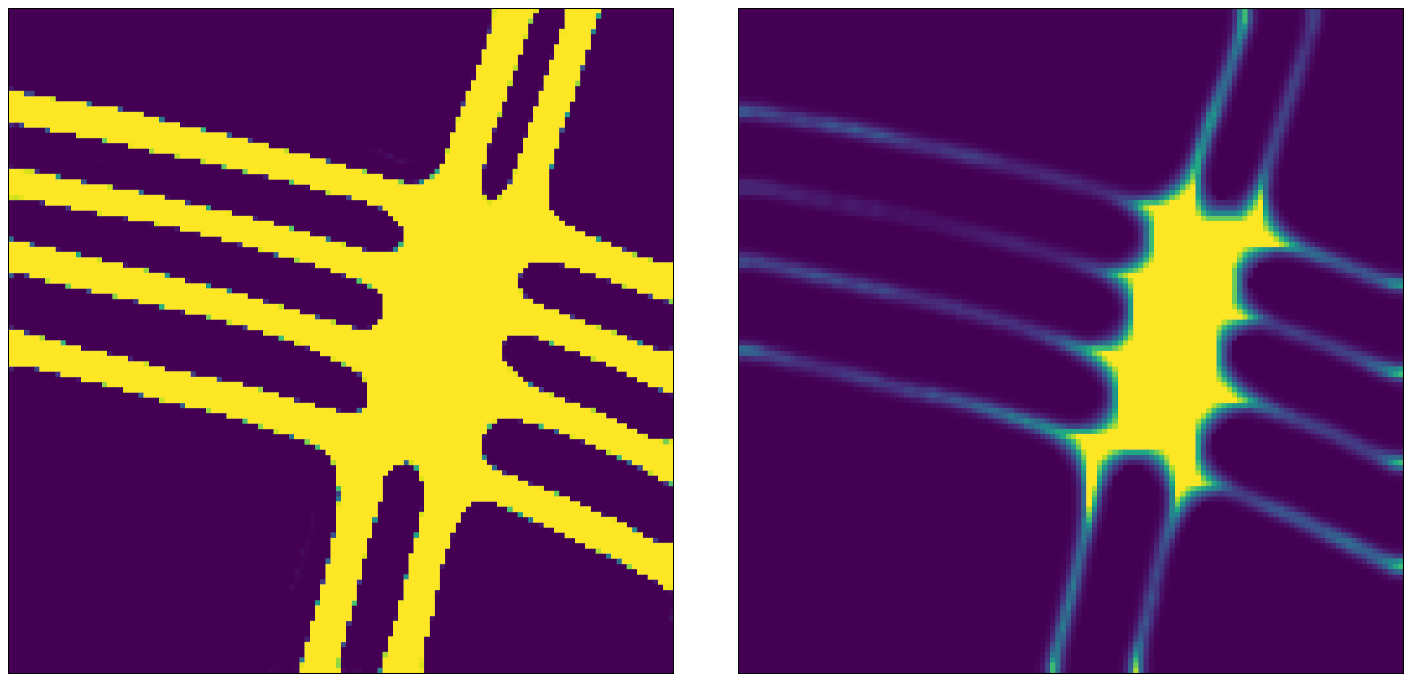}
\caption{To ease the unification of paths corresponding to the same road lane, the dense road lane affordance map output is smoothed, so that all optimal paths found by a search algorithm coincide with lane centers.}
\label{fig:smoothed_sla}
\end{figure}

\begin{algorithm}[t]
\label{alg:road_lane_network_graph_generation}
\SetAlgoLined
 \SetKwInOut{Input}{Input}
 \SetKwInOut{Output}{Output}
 
 \Input{Set of entry and exit points $Q_{entry}$, $Q_{exit}$, weighted adjacency matrix $A$}
 
 Initialize empty path sets $P_{entry}$, $P_{inter}$, $P_{exit}$, $P_{con}$
 
 \tcp{Find entry paths}
 
 \For{All entry points $q_{entry} \in Q_{entry}$}{
 
  $P_{entry\_tree} = \{\}$
  
  \For{All exit points $q_{exit} \in Q_{exit}$}{
  
   Path $p$ = A$^{*}$($q_{entry}$, $q_{exit}$, $A$)
   
   \If{$p$ exists}{
   
    Add $p$ to $P_{entry\_tree}$
   
   }
  }
  $p_{entry}$, $P^{*}_{con}$= unify($P_{entry\_tree}$)
  
  Add $p_{entry} \rightarrow P_{entry}$
  
  Add $P^{*}_{con} \rightarrow P_{con}$
 }
 
 \tcp{Find intersection, exit paths}
 
 \For{All exit points $q_{exit} \in Q_{exit}$}{
 
  $P_{exit\_tree} \leftarrow$ Paths in $P_{con}$ ending at exit point
  
  $p_{exit}$, $P^{*}_{inter}$= reverse\_unify($P_{exit\_tree}$)
  
  Add $p_{exit} \rightarrow P_{exit}$
  
  Add $P^{*}_{inter} \rightarrow P_{inter}$
  
 }
 
 Construct graph $G$ from $P_{entry}$, $P_{inter}$, $P_{exit}$

 \Output{Road lane network graph $G = (V, E)$}
 
 \caption{Road lane network graph generation}
\end{algorithm}

\subsection{Weighted adjacency matrix generation}
The weighted adjacency matrix $A$ defines the cost function for the A$^*$ search algorithm. The generation of $A$ starts with an 8-directional 2D grid map adjacency matrix that is initialized with infinite weight assigned to all edges. The edge weights $e_{A,B}$ are subsequently reset according to the dense affordance output as follows.

First, the dense road lane affordance output $y_{i,j}$ of each grid map point $(i,j)$ is thresholded to 0 and 1 and subsequently smoothed by an 8x8 kernel and transformed by the function$ f(\tilde{y}) = \tilde{y}^8$, so that optimal paths found by the search algorithm will coincide with lane centers. The effect of smoothing is displayed in Fig.~\ref{fig:smoothed_sla}. Next, for each gird map point $A$ at $(i,j)$, a neigboring point $B$ is reachable, if the direction $\overrightarrow{A \: B}$ is within an angle $\Delta \theta$ of any dense directional affordance components $\{\theta_1, \theta_2, \theta_3\}$ predicted by the neural model at point $(i,j)$. The edge weight to all reachable neighbors is reset according to~Eq.~(\ref{eq:edge_weight}).

\begin{equation}
\label{eq:edge_weight}
 e_{A,B} = \lvert \overrightarrow{A \: B} \rvert - \log{\tilde{y}_{B}}
\end{equation}

\subsection{Search-based path algorithm} The rest of this section explains the search-based algorithm which finds paths between road lane entry and exit points with $A$ as the cost function, and eventually decomposes the found paths into a road lane network graph $G$ corresponding to the formal model defined in the Appendix. First, a set of discrete entry and exit points $Q_{entry}$ and $Q_{exit}$ are computed from the dense entry and exit affordance map output of the neural model, determined as the center of momentum coordinate $(i,j)$ of the contour constituting each separable cluster of dense output.

The road lane network graph $G$ computation is summarized in Algorithm~\ref{alg:road_lane_network_graph_generation} and explained as follows. Three empty path sets $P_{entry}$, $P_{inter}$, $P_{exit}$ are initialized (line 1), denoting the set of entry, intersection, and exit paths, each corresponding to the equivalent set of entry, intersection, and exit lanes defined in the formal model. This decomposition is proved to be possible for any structured road scene.

For each entry point $q_{entry}$ (line 2), the paths to all reachable exit points $q_{exit}$ are found using the A$^*$ search algorithm with the weighted adjacency matrix $A$ as the cost function. All such paths are rooted at a single entry point and thus constitute a tree $P_{entry\_tree}$ (lines 3-7). The tree is unified from the root up until the first point where the tree diverges into separate branches, meaning all paths in the tree are modified to consist of the same entry path (line 8). The unification algorithm returns the common entry path $p_{entry}$ and a new set of connecting paths $P^{*}_{con}$ forming a tree rooted at the end of $p_{entry}$. These paths are added to the maintained sets (lines 9-10). Repeating these steps for all entrance points results in the set $P_{entry}$ constituting the paths of all entry lanes in the road scene. The metric used for tree divergence is the total angle $\Theta = \sum^{N-1}_{i=0}{\Delta \theta_{i,i+1}}$ spanned between all $N$ path directions arranged in counterclockwise order. The largest angle $\Delta \theta_{N-1, N}$ is excluded as it represents the angle \textit{not} spanned by the path directions. Path directionality is represented by a vector $\overrightarrow{p_{t} \: p_{t+\Delta t}}$ spanned by a point $p_{t}$ and a future point $p_{t + \Delta t}$ on the same path. This vector approximates the future directionality given a suitable number of lookahead steps $\Delta t$. In this work $\Delta t = 6$. 

The next part of Algorithm~\ref{alg:road_lane_network_graph_generation} applies the same process in the reverse direction. For each exit point $q_{exit}$ (line 11), all paths ending at $q_{exit}$ form a reverse path tree rooted at $q_{exit}$. After reversing all path directions, the same unification algorithm can be applied to unify the reverse tree from the end to the point where all paths converge into a single path, returning the common exit path $p_{exit}$ and a new set of paths $P^*_{inter}$ forming a reverse tree into the start of $p_{exit}$ (line 13). These paths are added to the maintained sets (lines 14-15).  Repeating these steps for all exit points results in the set $P_{exit}$ constituting the paths of all exit lanes in the road scene. Additionally, the remaining set of pruned paths $P_{inter}$ constitute all paths between connected entry and exit paths.

The road network graph $G$ is constructed by letting the first and last point of every path $p \in P_{entry}$ represent an entry and fork vertex sharing a directed edge. Similarly, the first and last point of every path $p \in P_{exit}$ represents a merge and exit vertex sharing a directed edge. The first and last point of all paths $p \in P_{inter}$ represents directed edges between a pair of fork and merge vertices, or in other words, intersection connectivity. The formal model is proved to be able to represent any structured road scene, and as $G$ is an arbitrary but particular instance of the formal model, by universal modus ponens, $G$ can represent any structured road scene given the dense affordance output by the neural model correctly represent the local structure of the road scene. Note that the method does not only infer connectivity between road lanes, but also infers intuitive spatial locations for fork and merge points, and consequently, the spatial extent of intersections.

\section{Experiments}
\label{sec:experiments}

\begin{figure}[t]
  \centering
  \includegraphics[width=0.45\textwidth]{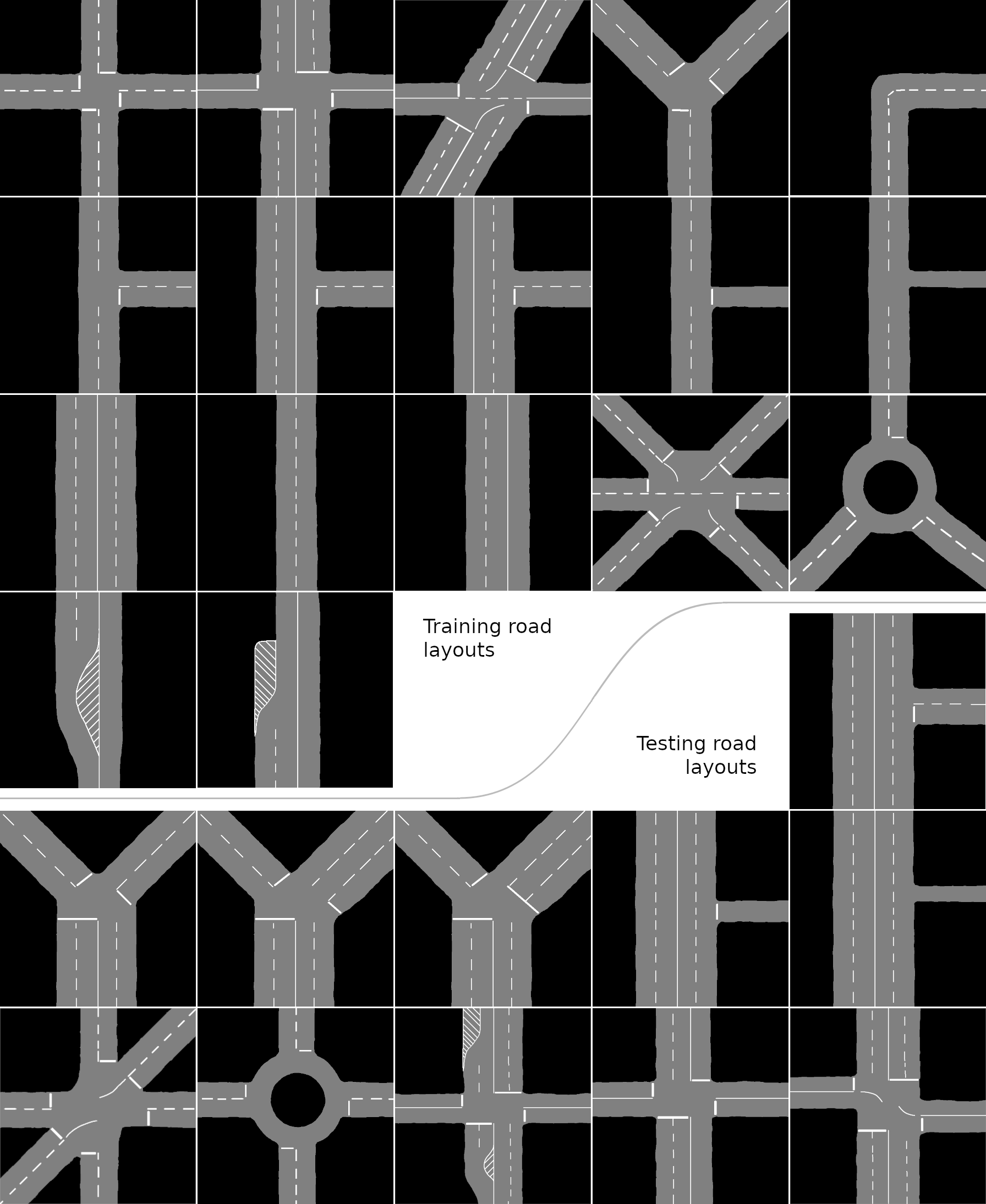}
  \caption{The upper group of 17 training road layouts are used for training and represent all common intersection layouts, encompassing the training distribution $p_{train}(x)$. The bottom 11 testing road layouts consist of variations of the upper group, encompassing the testing distribution $p_{test}(x)$. Ability to generalize to new intersections is measured by performance on $p_{test}(x)$. An actual inference and training sample is shown in Fig.~\ref{fig:gdsla_promo_img}.}
  \label{fig:road_layouts}
\end{figure}

\begin{figure*}[t]
  \includegraphics[width=\textwidth]{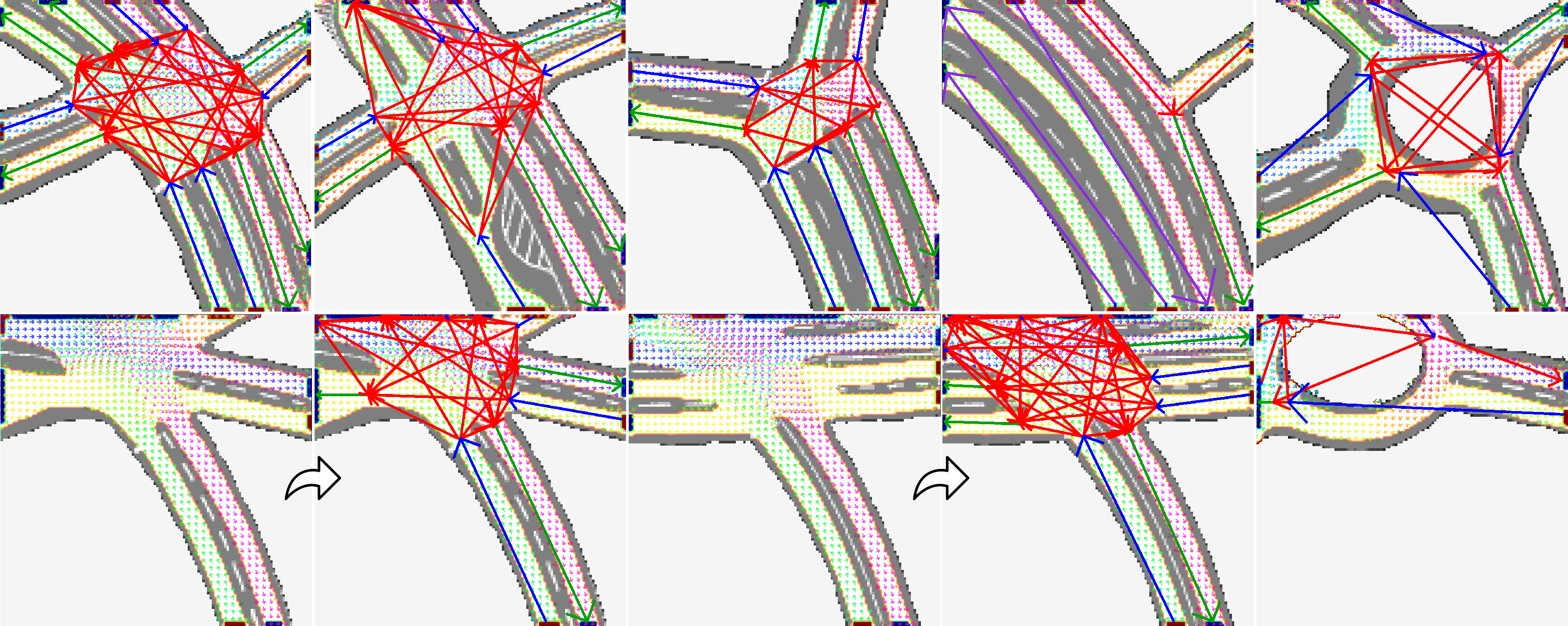}
  \caption{Test sample output visualizations. The top row shows visualizations of correctly inferred road lane network graphs with entry, intersection, exit, and non-intersection edges. The graphs are superimposed on the dense road lane affordance (white), entry and exit point affordance (red and blue), and directional affordance (small arrows). The bottom row demonstrates the method's robustness to partially observed intersection road scenes.}
  \label{fig:model_visualization}
\end{figure*}

\begin{figure}[t!]
  \centering
  \includegraphics[width=0.39\textwidth]{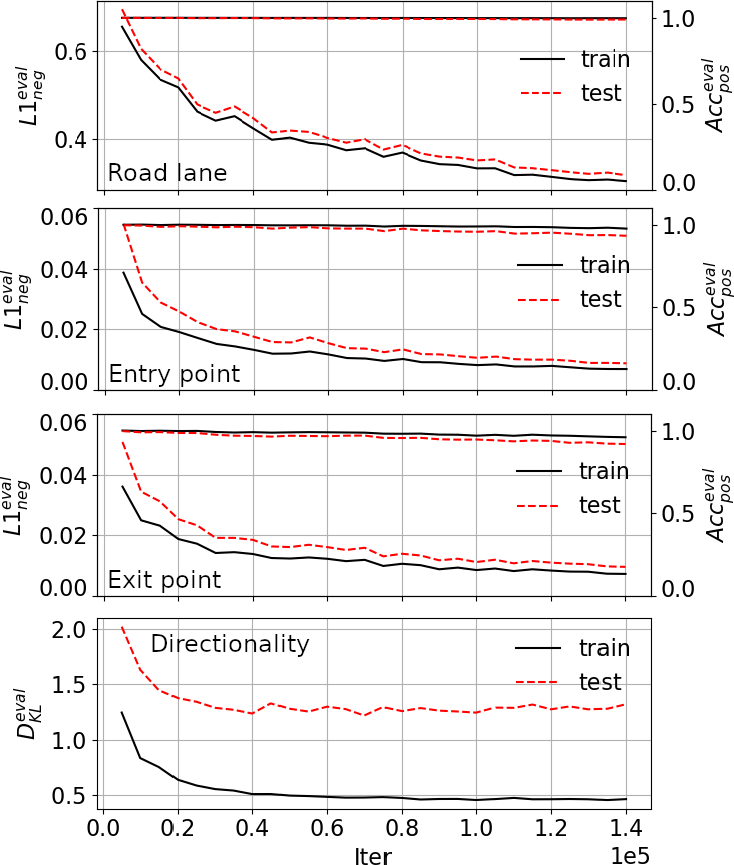}
  \caption{Neural model learning performance on the train and test road layout samples. The three upper plots show dense affordance prediction performance (i.e. road lane, entry, and exit point prediction). The upper nearly constant lines denote $Acc^{eval}_{pos}$ values. Dense affordance performance improves consistently with further training with test performance closely tracking training performance. The bottom plot shows directional affordance prediction indicating a larger gap between training and testing performance, which converges after 40,000 iterations. Metrics are explained in Sec~\ref{sec:performance_metrics}.}
  \label{fig:result_plots}
\end{figure}

The model is trained and tested on randomly generated samples drawn from two data distributions $p_{train}(x)$ and $p_{test}(x)$ representing the two sets of artificially created road layouts visualized in Fig.~\ref{fig:road_layouts}. Training samples are generated by randomly augmenting road layouts from $p_{train}(x)$ together with one example trajectory as visualized in Fig.~\ref{fig:gdsla_promo_img}. $p_{train}(x)$ is intended to cover all fundamental components of intersection layouts~\cite{MIRE}. Evaluation samples are generated from road layouts in both $p_{train}(x)$ and $p_{test}(x)$ by the same augmentation process, but with all valid trajectories superimposed. Generalization ability is evaluated by how well the model performs on samples drawn from $p_{test}(x)$, consisting of road layouts with similar components as in $p_{train}(x)$ but with new combinations, varying number of road lanes, and asymetric structure.

The neural model is trained on an Nvidia RTX 3090 GPU using stochastic gradient descent (SGD) with mini-batches of 28 samples and polynomial learning rate decay, reducing the initial learning rate 10$^{-3}$ to 10$^{-5}$ over 180,000 iterations with power 0.9.

\subsection{Data augmentation}

Data augmentation is a critical part of training the neural model to be invariant to particular road geometries as all training examples are unique, and thus improve generalization ability. This work follows the data augmentation approach presented in \cite{Karlsson2020} for generating randomly translated, rotated, and warped road scene samples.

\subsection{Performance metrics}
\label{sec:performance_metrics}

This work presents two new metrics $Acc^{eval}_{pos}$ and $L1^{eval}_{neg}$ to evaluate the performance of dense road lane, entry and exit point affordance predictions. Directional affordance $D^{eval}_{KL}$ is evaluated by computing KL divergence between the model output and evaluation label~\cite{Karlsson2020}. A set of 20 pregenerated samples per road layout is used for evaluation.

Desirable dense road lane affordance predictions have two requirements; all true positive elements should be predicted positive, and the less false positives the better. The former requirement can be represented by $Acc^{eval}_{pos}$ as the ratio of all true positive elements actually predicted positive (i.e. $y_{i,j} > 0.5$). The later requirement can be quantified by the linear L1 metric $L1^{eval}_{neg} = \sum^N_i (\lvert \hat{y_i} - y_i \rvert) / N $ for $i \in $ set of negative evaluation label elements.

The topological correctness of a road lane network graph $G$ is measured in terms of the number of missing or erroneous edges compared with the graph generated from the evaluation sample label.

\section{Results}
\label{sec:results}

Output visualizations of testing road layouts sampled from $p_{test}(x)$ are shown in Fig.~\ref{fig:gdsla_promo_img} and Fig.~\ref{fig:model_visualization}. The results demonstrate how the neural model has learned a generalized local representation of road scenes, which enables the graph algorithm to infer the global connectivity for new road scenes, including asymmetric intersections and one-way roads. In Fig.~\ref{fig:model_visualization} white heatmap regions represent predicted dense road lane affordance, The red and blue colored blobs represent predicted dense entry and exit points. The small arrows represents predicted local multimodal directionality. The road lane network graph is superimposed over the dense output, with entry edges colored blue, intersection edges colored red, exit edges colored green, and non-intersection lanes reduced to a single purple edge. Results are visualized by a neural model trained for 100,000 iterations, which takes about 26 hours on our hardware.

Fig.~\ref{fig:result_plots} provides a quantitative evaluation of training and generalization performance of the neural model output. The dense affordances learned by the barrier function loss formulation Eq.~(\ref{eq:barrier_function_loss}) results in comparable performance between train and test samples and continue to improve with further training iterations beyond 140,000 iterations, indicating strong generalization and thus learning of sound abstract representations of road scene structures. On the other hand, directional affordance performance converges relatively early after about 40,000 iterations and there is a larger gap between training and testing performance, both indicating that there is room for further improvements in directional modeling.

The result of road lane network graph evaluation is that 337/340 (99.1\%) training layouts are error free. Two errors are related to scenes with disconnected merging lanes. For testing layouts, 199/220 (90.5\%) samples are error free. 14 failure cases are related to inconsistent directional affordance in asymmetric intersection layouts, such as one-way roadways and lanes with ambiguous directionality. The other 7 failures are for the forking lane intersection where the forked road lane is disconnected or otherwise improperly masked out and obstructs the search algorithm.

As explained in Sec.~\ref{sec:automatic_lane_level_map_generation}, our proposed method is the first self-supervised approach for road lane network learning, and extending prior methods to become self-supervised is nontrivial, making direct comparison difficult.

\section{Conclusions}
\label{sec:conclusions}

This work presents the first self-supervised method for learning a hybrid neural and search-based model to infer a spatially grounded graphical representation of the road lane network, based on a dense segmented input representation of the road scene generatable from sensor data as input. The model is formally proved to be the most compressed representation of the road lane network while retaining the notion of intersection regions. Training data can be automatically and cost-effectively generated from driving data. Data augmentation enables the model to learn invariance to particular geometries and instead learn representations of general structures of road layouts. Experiments demonstrate the model successfully generalizing to new intersection layouts not encountered during training, therefore advancing the state of the art in automatic road lane network modeling.

Experiment results demonstrate the potential of our approach to enable conventional HD map dependent AV motion planning and prediction systems to operate in the absence of human annotated HD map information, by utilizing artificial intelligence to generate a discrete graphical representation of the road lane network based on onboard sensor data.

Future work includes adding a symbolic learning and reasoning step to refine the road lane network graph and resolve ambiguity associated with highly asymmetric road scenes, as well as refine edges to conform with explicit and implicit traffic rules, such as considering traffic signs and generally prohibiting U-turns. The method is to be demonstrated using real-world data. Previous works \cite{Salzmann2019, Amini2019, PerezHigueras2018, Barnes2016, Wang2019} have demonstrated that CNN-based navigation model approaches can work with noisy and partially occluded real-world input representations of urban road scenes, and thus it is reasonable to expect our method and experimental results to also transfer to real-world inputs.

\section*{APPENDIX: Formal road lane network model}
\label{appendix:formal_road_lane_network_model}

The road lane network model is a formal model to represent the road lane connectivity of any structured road scene. The model is formally proved to be the most compressed representation retaining the notion of intersection regions.

\begin{definition}[Road scene]
 The approximate representation of the external world, to the extent relevant for the navigation task, is defined as the road scene~\cite{Ulbrich2015}. This representation includes a description of the spatial configuration and semantics of static and dynamic elements constituting the scene, such as drivable regions and road markings.
\end{definition}

\begin{definition}[Road lane network]
 The directional connectivity between lanes in a road scenes can be mathematically represented as directed graph $G(V, E)$. The set of vertices $V = \{ v_1 , \dots, v_N\}$ constitute structurally representative points where lanes enter and exit the road scene (i.e. \textit{entry} and \textit{exit vertices}), as well as internal points associated with intersection. The set of directed edges $E = \{ e_1, \dots, e_M \}$ represents the connectivity between the structural points.
\end{definition}

\begin{definition}[Intersection]
 An intersection or junction denotes a part of the road scene where incoming road lanes branch and connect with lanes from other entry points. In other words, intersections present a navigational branching choice. All road lanes exiting the intersection at the same point are merged into a single exit lane.
 
 Mathematically, an intersection is a subgraph $H$ of the road lane network graph $G$ consisting of at least one \textit{fork vertices} with outdegree $deg^{-}{(v)} \ge 2$ (i.e. two or more outgoing edges) or \textit{merge vertices} with indegree $deg^{+}(v) \ge 2$ (i.e. two or more incoming edges). If both fork and merge vertices exist, then for every fork vertex there exists at least one edge to some merge vertex and vice versa. A road lane network can contain multiple intersections if and only if the intersections are disjoint subgraphs.
\end{definition}

\begin{lemma}
\label{lemma:one_intersection}
 Any road scene can be spatially reduced so that every road lane is part of only one intersection while still being a sufficient representation for the navigation task.
\end{lemma}

\begin{proof}
 A human driver perceiving the road one intersection at a time can successfully navigate any road lane of any road scene. Therefore, by modus ponens, a sufficiently intelligent AV system utilizing a road scene representation encompassing only a single intersection per road lane can also successfully navigate any road scene.
\end{proof}

\begin{lemma}
 \label{lemma:dag_depth_3}
 The road lane network for any road scene can be represented by a directed acyclic graph (DAG) of at most depth three.
\end{lemma}

\begin{proof}
 Suppose a directed graph $G = (V, E)$ represents the road lane network graph of a particular but arbitrary road scene.
 
 By the definition of a road lane network, all valid road lanes begin at an entry node $v_{entry}$ and end at an exit node $v_{exit}$. By Lemma~\ref{lemma:one_intersection}, all valid road lanes can be a part of only one intersection. By the definition of an intersection, each incoming road lane is associated with one fork node $v_{fork}$ and a merge node $v_{merge}$, or only one of the nodes. Thus, any valid road lane originating from any road scene entry point can be expressed by at most four vertices, connected by three directional edges.
 
 Additionally, as every child vertex is strictly closer to an exit vertex than its parent vertex, every path through the graph decreases the distance to an exit vertex, and the graph is acyclic. Therefore, the directed graph $G$ has a depth of at most three edges and is acyclic.
\end{proof}

\begin{corollary}
 Any road lane that is part of an intersection can be decomposed into three edges; an \textit{entry edge} from an entry vertex to a fork vertex, an \textit{intersection edge} from a fork vertex to a merge vertex, and an \textit{exit edge} from a merge vertex to an exit vertex.
\end{corollary}

\begin{corollary}
 A road lane not part of an intersection can be represented by a single directional edge.
\end{corollary}

\begin{corollary}
 A special case of intersections is the branching of a single road lane, or merging of two road lanes, constituting a \textit{point intersection} represented by a DAG of depth two, which is the most compact representation.
\end{corollary}

\begin{corollary}
 Road scenes with an intersection consisting of two or more incoming and outgoing road lanes can be represented by a DAG of depth three, which is the most compact representation.
\end{corollary}

\section*{ACKNOWLEDGMENT}

This research was supported by Program on Open Innovation Platform with Enterprises, Research Institute and Academia, Japan Science and Technology Agency (JST, OPERA, JPMJOP1612).


\begin{thebibliography}{99}

\bibitem{Paden2016} B. Paden, M. Cap, S. Yong, D. Yershov, and E. Frazzoli, A Survey of Motion Planning and Control Techniques for Self-driving Urban Vehicles, IEEE Transactions on Intelligent Vehicles, 2016.
\bibitem{Claussmann2019} L. Claussmann, M. Revilloud, D. Gruyer, and S. Glaser, A Review of Motion Planning for Highway Autonomous Driving, IEEE Transactions on Intelligent Transportation Systems, 2019.
\bibitem{Sheif2016} H. Sheif, and X. Hu, Autonomous driving in the iCity-HD maps as a key challenge of the automotive industry, Engineering, 2016.
\bibitem{Kuutti2018} S. Kuutti, S. Fallah, K. Katsaros, M. Dianati, F. Mccullough, and A. Mouzakitis, A Survey of the State-of-the-Art Localization Techniques and Their Potentials for Autonomous Vehicle Applications, IEEE Internet of Things Journal, 2018.
\bibitem{Wolcott2017} R. Wolcott, and R. Eustice, Robust LIDAR localization using multiresolution Gaussian mixture maps for autonomous driving, The International Journal of Robotics Research, 2017.
\bibitem{Levinson2010} J. Levinson, and S. Thrun, Robust vehicle localization in urban environments using probabilistic maps, IEEE International Conference on Robotics and Automation, 2010.
\bibitem{Urmson2009} C. Urmson, et al., Autonomous Driving in Urban Environments: Boss and the Urban Challenge, Journal of Field Robotics, 2009.
\bibitem{Nilsson2016} J. Nilsson, J. Silvlin, M. Brannstrom, E. Coelingh, and J. Fredriksson, If, When, and How to Perform Lane Change Maneuvers on Highways, IEEE Intelligent Transportation Systems Magazine, 2016.
\bibitem{Bansal2018} M. Bansal, A. Krizhevsky, and A. Ogale, ChauffeurNet: Learning to Drive by Imitating the Best and Synthesizing the Worst, Robotics: Science and Systems, 2019.
\bibitem{Schulz2018} J. Schulz, C. Hubmann, J. L{\"o}chner, and D. Burschka, Multiple Model Unscented Kalman Filtering in Dynamic Bayesian Networks for Intention Estimation and Trajectory Prediction, International Conference on Intelligent Transportation Systems (ITSC), 2018.
\bibitem{Djuric2020} N. Djuric, V. Radosavljevic, H. Cui, T. Nguyen, F. Chou, T. Lin, N. Singh, and J. Schneider, Uncertainty-aware Short-term Motion Prediction of Traffic Actors for Autonomous Driving, IEEE Winter Conference on Applications of Computer Vision (WACV), 2020.
\bibitem{Salzmann2019} T. Salzmann, J. Thomas, T. Kuehbeck, J. Sung, S. Wagner, and A. Knoll, Online Path Generation from Sensor Data for Highly Automated Driving Functions, IEEE Intelligent Transportation Systems Conference (ITSC), 2019.
\bibitem{Iesaki2019} H. Iesaki, S. Naruse, T. Hirakawa, T. Yamashita, and H. Fujiyoshi, Automatic Creation of Path Information on Digital Map, IEEE Intelligent Transportation Systems Conference (ITSC), 2019.
\bibitem{Zhao2019} J. Zhao, X. He, J. Li, T. Feng, C. Ye, and L. Xiong, Automatic Vector-Based Road Structure Mapping Using Multibeam LiDAR, IEEE International Conference on Intelligent Transportation Systems (ITSC), 2018.
\bibitem{Guo2016} C. Guo and K. Kidono and J. Meguro and Y. Kojima and M. Ogawa and T. Naito, A Low-Cost Solution for Automatic Lane-Level Map Generation Using Conventional In-Car Sensors, IEEE Transactions on Intelligent Transportation Systems, 2016.
\bibitem{Homayounfar2018} N. Homayounfar, W. Ma, S. Lakshmikanth, and R. Urtasun, Hierarchical Recurrent Attention Networks for Structured Online Maps, IEEE/CVF Conference on Computer Vision and Pattern Recognition (CVPR), 2018.
\bibitem{Homayounfar2019} N. Homayounfar, J. Liang, W. Ma, J. Fan, X. Wu, and R. Urtasun, DAGMapper: Learning to Map by Discovering Lane Topology, IEEE/CVF International Conference on Computer Vision, 2019.
\bibitem{Amini2019} A. Amini, G. Rosman, S. Karaman, and D. Rus, Variational End-to-End Navigation and Localization, International Conference on Robotics and Automation (ICRA), 2019.
\bibitem{Ort2018} T. Ort, L. Paull, and D. Rus, Autonomous Vehicle Navigation in Rural Environments Without Detailed Prior Maps, IEEE International Conference on Robotics and Automation (ICRA), 2018.
\bibitem{PerezHigueras2018} N. P\'{e}rez-Higueras, F. Caballero, and L. Merino, Learning Human-Aware Path Planning with Fully Convolutional Networks, IEEE International Conference on Robotics and Automation (ICRA), 2018.
\bibitem{Barnes2016} D. Barnes, W. Maddern, and I. Posner, Find your own way: Weakly-supervised segmentation of path proposals for urban autonomy, IEEE International Conference on Robotics and Automation (ICRA), 2016.
\bibitem{Wang2019} Z. Wang, B. Liu, S. Schulter, and M. Chandraker, A Parametric Top-View Representation of Complex Road Scenes, IEEE/CVF Conference on Computer Vision and Pattern Recognition (CVPR), 2018.
\bibitem{Kunze2018} L. Kunze, T. Bruls, T. Suleymanov, and P. Newman, Reading between the Lanes: Road Layout Reconstruction from Partially Segmented Scenes, IEEE Intelligent Transportation Systems Conference (ITSC), 2018.
\bibitem{Geiger2014} A. Geiger, M. Lauer, C. Wojek, C. Stiller, and R. Urtasun, 3D Traffic Scene Understanding From Movable Platforms, IEEE Transactions on Pattern Analysis and Machine Intelligence, 2014.
\bibitem{Karlsson2020} R. Karlsson, and E. Sjoberg, Learning a Directional Soft Lane Affordance Model for Road Scenes Using Self-Supervision, IEEE Intelligent Vehicles Symposium (IV), 2020.
\bibitem{Prokudin2018} S. Prokudin, P. Gehler, and S. Nowozin, Deep Directional Statistics: Pose Estimation with Uncertainty Quantification, ECCV, 2018.
\bibitem{Nava2019} M. Nava, J. Guzzi, R. Chavez-Garcia, L. Gambardella, and A. Giusti, Learning Long-Range Perception Using Self-Supervision from Short-Range Sensors and Odometry, IEEE Robotics and Automation Letters, 2019.
\bibitem{Manwani2013} N. Manwani and P. S. Sastry, Noise Tolerance Under Risk Minimization, IEEE Transactions on Cybernetics, 2013.
\bibitem{Mitchell1982} T. Mitchell, Generalization as search, Artificial Intelligence, 1982.
\bibitem{Hart1968} P. Hart, N. Nilsson, and B. Raphael, A Formal Basis for the Heuristic Determination of Minimum Cost Paths, IEEE Transactions on Systems Science and Cybernetics, 1968.
\bibitem{MIRE} Model Inventory of Roadway Elements MIRE 2.0., U.S. Department of Transportation, Federal Highway Administration, 2017.
\bibitem{Ulbrich2015} S. Ulbrich, T. Menzel, A. Reschka, F. Schuldt, and M. Maurer, Defining and Substantiating the terms Scene, Situation, and Scenario for Automated Driving. IEEE Intelligent Transportation Systems Conference (ITSC), 2015.

\end{thebibliography}
\end{document}